\documentclass{ifacconf}
\usepackage{amsmath,amssymb,amsfonts,mathtools}
\usepackage{algorithm}


\usepackage{algorithm} 
\usepackage[noend]{algpseudocode} 
\newcommand{\LFM}{L_{\mathrm{FM}}}
\DeclareMathOperator{\KL}{KL}
\newcommand{\E}{\mathbb{E}}
\newcommand{\N}{\mathcal{N}}
\newcommand{\Pmod}{\mathcal{P}^{\theta}_\tau}
\newcommand{\Tpush}[1]{{#1}_{\#}}\newcommand{\cd}[1]{(\cdot\,|\,#1)}
\usepackage{subcaption}

\newtheorem{theorem}{Theorem}
\newtheorem{remark}{Remark}

\usepackage{graphicx} 
\usepackage{natbib} 

\usepackage[deletedmarkup=uwave, defaultcolor=cyan]{changes}

\usepackage{color,soul}

\newcommand{\IR}{{\mathbb{R}}}
\newcommand{\TP}{^\mathsf{T}}

\newcommand{\dd}{\mathrm{d}}

\usepackage{letltxmacro}
\LetLtxMacro\orgvdots\vdots
\LetLtxMacro\orgddots\ddots

\makeatletter
\DeclareRobustCommand\vdots{%
  \mathpalette\@vdots{}%
}
\newcommand*{\@vdots}[2]{%
  \sbox0{$#1\cdotp\cdotp\cdotp\m@th$}%
  \sbox2{$#1.\m@th$}%
  \vbox{%
    \dimen@=\wd0 %
    \advance\dimen@ -3\ht2 %
    \kern.5\dimen@
    \dimen@=\wd2 %
    \advance\dimen@ -\ht2 %
    \dimen2=\wd0 %
    \advance\dimen2 -\dimen@
    \vbox to \dimen2{%
      \offinterlineskip
      \copy2 \vfill\copy2 \vfill\copy2 %
    }%
  }%
}
\DeclareRobustCommand\ddots{%
  \mathinner{%
    \mathpalette\@ddots{}%
    \mkern\thinmuskip
  }%
}
\newcommand*{\@ddots}[2]{%
  \sbox0{$#1\cdotp\cdotp\cdotp\m@th$}%
  \sbox2{$#1.\m@th$}%
  \vbox{%
    \dimen@=\wd0 %
    \advance\dimen@ -3\ht2 %
    \kern.5\dimen@
    \dimen@=\wd2 %
    \advance\dimen@ -\ht2 %
    \dimen2=\wd0 %
    \advance\dimen2 -\dimen@
    \vbox to \dimen2{%
      \offinterlineskip
      \hbox{$#1\mathpunct{.}\m@th$}%
      \vfill
      \hbox{$#1\mathpunct{\kern\wd2}\mathpunct{.}\m@th$}%
      \vfill
      \hbox{$#1\mathpunct{\kern\wd2}\mathpunct{\kern\wd2}\mathpunct{.}\m@th$}%
    }%
  }%
}
\makeatother
 
\begin{document}
\begin{frontmatter}

\title{Perron--Frobenius Operator Matching for Generative Modeling%
\thanksref{footnoteinfo}}

\thanks[footnoteinfo]{This work was supported in part by the Hong Kong RGC under
Project CityU 11203321, CityU 11213322, CityU 11207823. XQ acknowledges the support from U.S. National Science Foundation~(NSF) grants SHF-2215573 and IIS-2212419.}

\author[First]{Shiqi Zhang}
\author[Second]{Wuwei Wu}
\author[First]{Jaemin Oh}
\author[Second]{Jie Chen}
\author[First]{Xiaoning Qian}

\address[First]{Texas A\&M University, College Station, TX 77840, USA (e-mail: shiqizhang001@tamu.edu; jaemin\_oh@tamu.edu; xqian@ece.tamu.edu)}
\address[Second]{City University of Hong Kong, Kowloon, Hong Kong SAR (e-mail: w.wu@my.cityu.edu.hk; jichen@cityu.edu.hk)}

\begin{abstract} 
We introduce Perron--Frobenius Operator Matching (PFOM), a generative framework that matches density evolution via the integral PF operator, subsuming flow, diffusion, and jump models. We prove that among Bregman divergences, only Kullback--Leibler divergence preserves equality between density-level and sample-conditioned objectives, yielding a practical loss equivalent to Koopman path matching. We further develop Nesterov-accelerated training and sampling that stabilize discretization and accelerate convergence. 
PFOM achieves faster KL/$W_2$/MMD decrease and improved wall-clock efficiency with empirical validation. PFOM unifies operator-theoretic identification with modern generative modeling and opens paths to adaptive dictionaries and high-dimensional applications.
\end{abstract}


\end{frontmatter}
\section{Introduction}
Characterizing Markov processes is fundamental to stochastic analysis \citep{ross1995stochastic}, with wide-ranging applications in, e.g., finance \citep{rolski2009stochastic}, statistical physics \citep{van1992stochastic}, and signal processing \citep{oppenheim1997signals}. The recent surge of artificial intelligence and generative modeling has amplified interest in learnable Markovian dynamics \citep{ho2020denoising,yang2023diffusion,lipman2024flow}, which is of interest to modeling and control of large-scale, complex systems, especially in the context of neural network-based control design \citep{katz2022verification} and
generative AI-driven automated control algorithms \citep{cui2025gencontrol}.

A central challenge is to efficiently and accurately parameterize Markov processes. Operator-theoretic perspectives provide a principled route: the Markov transfer operator \citep{eisner2015operator} offers a dominant characterization, and for (nonlinear) semidynamical systems, Perron--Frobenius theory grounds the Markov semigroup \citep{lemmens2012nonlinear,lasota2013chaos}, revealing the duality between Koopman and Perron--Frobenius operators. From a control-theoretic viewpoint, these operators encode the probabilistic evolution of closed-loop dynamics under stochastic policies and exogenous disturbances, enabling linear surrogates for stability analysis, constraint satisfaction, and performance verification of nonlinear systems. In safety-critical applications such as robotics, power systems, and networked infrastructures, learning and manipulating such operators from data is therefore crucial for risk-aware decision-making and robust control synthesis. Building on this view, data-driven identification methods, such as DMD \citep{proctor2016dynamic} and EDMD \citep{li2017extended,brunton2016discovering}, have become standard.

Concurrently, modern generative models, such as diffusion \citep{ho2020denoising,yang2023diffusion} and flow-based models \citep{lipman2022flow,lipman2024flow}, impose stronger demands: capturing multimodality and nonlinear density evolution with sample-conditioned efficiency. Traditional Koopman/Perron--Frobenius identification, designed primarily for prediction and control, does not directly address these generative objectives.

To bridge this gap, we introduce \emph{Perron--Frobenius Operator Matching (PFOM)}. PFOM (i) generalizes diffusion and flow matching paradigms by matching full density evolution—extending beyond first-order (velocity) descriptions to infinitely many orders—and (ii) strengthens operator learning for generative purposes by aligning density-level objectives with sample-conditioned criteria, thereby unifying operator-theoretic identification with modern generative modeling.

An important extension of PFOM is an \emph{inertial} optimization/sampling scheme based on Nesterov’s acceleration \citep{nesterov2018lectures}. We employ a lookahead extrapolation on the operator-parameter iterates and an inertial update on sample trajectories. Concretely, PFOM alternates between (a) extrapolated evaluation of the PF loss at a momentum point and (b) corrective updates with restart/monotone safeguards. This yields:
(i) faster empirical convergence of the PF loss and density metrics (KL, $W_2$, MMD);
(ii) reduced discretization error in sample propagation due to lookahead stabilization.

The rest of the paper is organized as follows: In Section \ref{prel}, we review relevant background knowledge of Koopman/Perron--Frobenius theory, Wasserstein and Bregman divergence measures, and generative modeling. In Section \ref{PFOM}, we explain why and how (under what measure) we should look at Perron--Frobenius operator matching, and then we convert it into the Koopman path matching problem for implementation. Section \ref{Nesterov} brings up a Nesterov momentum accelerating method for faster generation. Section \ref{simu} demonstrates with simulations and Section \ref{conclu} concludes the paper.

\section{Preliminaries}\label{prel}

\subsection{Koopman and Perron--Frobenius Operators}
Consider a nonlinear dynamical system
$x_{t} = S_t(x_0),$
where $S\colon\mathbb{R}^n\to \mathbb{R}^n$ is a non-singular mapping.
For some $f\in {L}_\infty$, the Koopman operator $\mathcal{K}_\tau$, is defined as
\begin{align}
(\mathcal{K}_\tau f)(x_t) = f(S_\tau(x_t)).
\end{align}
For some $g\in {L}_1$, the Perron--Frobenius (PF) operator $\mathcal P_\tau$, is defined as
\begin{align}\label{pf}
 \int_{y\in A} (\mathcal P_\tau g)(y) \dd y = \int_{x\in S_\tau^{-1}(A)} g(x) \dd x, \quad \forall A\in \Sigma,
\end{align}
where $\Sigma$ denotes some $\sigma$-algebra corresponding to the space $\mathbb{R}^n$.
 When $g$ is a density function, PF operator $\mathcal P_\tau$ is actually a \textit{Markov operator} that pushes forward present density to future densities \citep{lasota2013chaos}.

By PF-Koopman duality, for some $f\in {L}_\infty$ and some density $g\in {L}_1$, we always have \begin{align}\label{eq:duality}
 \langle \mathcal{K}_\tau f,g\rangle=&\langle f,\mathcal{P}_\tau g\rangle.
\end{align}
This means that the Koopman and PF operators form a dual pair.


\subsection{Generative Modeling} 

Consider two random vectors $X_0\sim\mathcal{N}(0,I)$ and $X_1\sim q(X_1)$, where $X_0$ is generated from the known prior distribution while $X_1$ is from some distribution $q(X_1)$ whose analytical form is not known \textit{a priori}. The objective for \textbf{generative modeling} is to learn a generative model $\mathcal{M}_\theta(X_1)$, from observed data $\mathcal D(X_1)$, to generate samples following the distribution $q(X_1)$. 

As shown in Fig. \ref{flow}, 
one of such generative modeling strategies is \textbf{Flow Matching} (FM) \citep{lipman2022flow}. It constructs a probability
path $(p_{t})_{t\in[0,1]}$, from a known source distribution $p_0 = p$ to the target distribution $p_1 = q$, where each $p_t$
is a distribution over $\mathbb R^d$. Specifically, FM adopts a simple regression objective to train the velocity field neural
network describing the instantaneous velocities of samples—later used to convert the source distribution $p_0$
into the target distribution $p_1$, along the probability path $p_t$.
That is, minimizing the flow matching loss:
\begin{align}
\label{unconditional}
 \mathbb E_{X_t\sim p_t;t\in \mathrm{Unif}[0,1]}{\big\|u(X_t)-v_\theta(X_t)\big\|}^2
\end{align} by minimizing its surrogate version (the one conditional on $X_0$ and $X_1$):
\begin{multline}\label{conditional}
 \mathbb E_{X_0\sim p X_1\sim q,t\sim \mathrm{U}[0,1]}\Big\|u(X_t(X_1,X_0,t)) \\ - v_\theta(X_t(X_1,X_0,t))\Big\|^2.
\end{multline}
Notice that for \eqref{conditional} and $\eqref{unconditional}$ to have the same optima, one has to use the \textbf{ Sample-Level Bregman divergence} as a distance measure, of which the mean squared error (MSE) loss is a special choice.
After training, we generate a novel sample from the target distribution $X_1 \sim q$ by (i) drawing a novel sample from the source distribution $X_0 \sim p$, and (ii)
solving the ordinary differential equation (ODE) determined by the velocity field:
$
 \dot X_t = v_\theta(X_t), \quad t\in[0,1].
$

In the discrete time settings, FM is formulated as \textbf{Path Matching}. Meanwhile, the flow ODE $\dot X_t = v_\theta(X_t)$ is solved by simulating the discrete path equation $X_{k+1} = {X_{k} + \tau } v_\theta(X_k)$.
\begin{figure}[tb]
\centerline{\includegraphics[width=0.25\columnwidth]{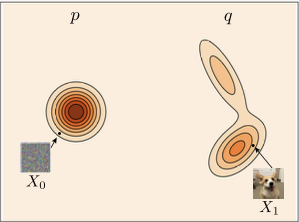}\\\includegraphics[width=0.25\columnwidth]{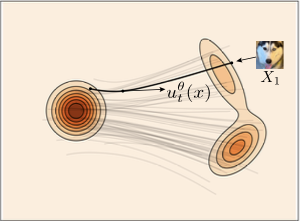}}
\vspace*{-1ex}
\caption{\small Demonstration for sample and noise (left) and the corresponding generation process (right) \citep{lipman2024flow}}
\label{flow}
\end{figure}

Flow matching and diffusion models control the \emph{local} terms in \eqref{eq:generator}—drift (first order) and diffusion (second order)—within a KFE-based objective \citep{lipman2022flow,lipman2024flow}. Recent generator matching extends this to include jump contributions \citep{holderrieth2024generator}. However, all these existing formulations operate at the \emph{infinitesimal} level and characterize only the first-, second-order, or jump terms of the differential approximation, which may fail to capture higher-order, multi-step transport effects crucial for complex, multi-modal density evolution.

\section{Perron-Frobenius Operator Matching}\label{PFOM}

{We here propose a new generative modeling framework, Perron--Frobenius operator matching (PFOM), }
which elevates generative modeling from matching local, infinitesimal dynamics to directly aligning the finite-time evolution of densities. Instead of constraining only the drift/diffusion terms in a Kolmogorov forward equation~(KFE)—as in flow matching and diffusion models—PFOM works at the level of the integral Perron--Frobenius operator, which encapsulates the full Markov semigroup, considering higher-order and multi-step transport effects that are critical for complex, multimodal distributions. By matching $\mathcal{P}_\tau \rho_t$ and $\rho_{t+\tau}$ at a finite step $\tau$,
 PFOM captures richer global evolution than purely velocity-based schema, while still remaining compatible with operator-theoretic tools such as Koopman and DMD/EDMD-based identification.

We further formulate PFOM with a practical, sample-conditioned training loss. We show that among separable Bregman divergences, KL is the unique choice that keeps the density-level PF loss exactly aligned with its conditional counterpart, thereby justifying a KL-based PFOM objective for generative training. Pushing this loss through the PF--Koopman duality yields an equivalent Koopman path-matching formulation that can be realized with neural operators or classical DMD/EDMD~\citep{proctor2016dynamic,li2017extended}. We 
derive Nesterov-style inertial updates for faster and more stable optimization and sampling. In
this section we first formalize PFOM, establish its Koopman equivalences, and introduce a Nesterov-accelerated variant tailored for efficient training. 

\subsection{Why Perron--Frobenius Operators?}

Let $(\mathcal{P}_\tau)_{\tau\ge 0}$ denote the Perron--Frobenius (PF) semigroup acting on densities, and $(\mathcal{K}_\tau)_{\tau\ge 0}$ the Koopman semigroup acting on test functions (observables). A sufficiently regular Markov process with drift $u_t(x)$, diffusion $\sigma_t(x)$, and jumps, is governed by the KFE \citep{risken1989fokker,lasota2013chaos}:
$
\label{kfe}
\partial_t\langle \rho_t, f\rangle \;=\; \langle \rho_t, \mathcal {L}^* f\rangle,
$
where $\mathcal{L}{^*}$ is the (Koopman) infinitesimal generator acting on $f$:
\begin{multline}\label{eq:generator}
\mathcal {L}^*f(x) = \underbrace{{u_t(x)\!}\TP\nabla f(x)}_{\text{drift}}
+ \underbrace{\tfrac12\,\mathrm{tr}\big(\sigma_t(x){\sigma_t(x)\!}\TP \nabla^2 f(x)\big)}_{\text{diffusion}}
\\+ \underbrace{\text{(jump term)}}_{\text{if present}}.
\end{multline}
Equivalently, on densities the adjoint generator $\mathcal{L}$ yields the Fokker--Planck form $\partial_t \rho_t = \mathcal{L} \rho_t$. The integral PF operator satisfies
\begin{align}
\rho_{t+\tau} \;=\; \mathcal{P}_\tau \rho_t \;=\; e^{\tau \mathcal{L}} \rho_t,
\qquad
\langle \rho_{t+\tau}, f\rangle \;=\; \langle \rho_t, \mathcal{K}_\tau f\rangle,
\end{align}
so that $\mathcal{K}_\tau = e^{\tau \mathcal {L}^*}$ and $\mathcal{P}_\tau = e^{\tau \mathcal{L}}$ are dual.


In contrast, PFOM compares the \emph{integral} evolution
$
\mathcal{P}_\tau \rho_t \;\;\text{vs.}\;\; \rho_{t+\tau}
$
for finite $\tau$, thereby capturing \emph{all orders} in the expansion of $e^{\tau \mathcal{L}}$ \citep{risken1989fokker}. Practically, this allows us to train against richer, multi-step transport phenomena that are invisible to purely infinitesimal matching.



\subsection{Wasserstein-Divergence Guided PFOM}
We denote $\Pi(\rho_0,\rho_1)$ as the set of all possible joint distributions with starting marginal density $\rho_0$ and ending marginal density $\rho_1$.
The \textbf{Wasserstein-2 metric} is defined by:
\begin{equation}
W_2^2(\rho_0,\rho_1) =\inf_{\pi\in\Pi(\rho_0,\rho_1)} \int_{\IR^n\times \IR^n}{\|x-y\|}^2 \pi(\dd x,\dd y).
\end{equation}
In~\citet{karimi2022data}, the authors took the Wasserstein-2 metric as the loss function to
match the density flow $\rho_k(x)$ through learning the Perron--Frobenius operator $ \mathcal{P}$ such that $(\mathcal P \rho_k) (x) =\rho_{k+1}(x)$: 
\begin{multline*}
W_2^2((\mathcal{P} \rho_k)(x),\rho_{k+1}(y))\\=\inf_{\pi\in\Pi(\mathcal{P} \rho_k,\rho_{k+1})}\int dy dx\|x-y\|^2\pi(x,y).
\end{multline*}

Consider a set of observables $\{\phi_k\}_{k=1}^K \subset L_\infty$ and define the dictionary $\Phi\coloneqq \IR^n\to\IR^K$ as the vector-valued function
$
    \Phi(x) = {\begin{bmatrix}
        \phi_1(x) & \cdots & \phi_K(x)
    \end{bmatrix}}\TP.
$
The Koopman operator $\mathcal{K}_{\tau}$ acts on this dictionary component-wise, yielding
$
\mathcal{K}_{\tau} \Phi = {\begin{bmatrix}
        \mathcal{K}_{\tau} \phi_1 & \cdots & \mathcal{K}_{\tau} \phi_K
    \end{bmatrix}}\TP.$
The discrepancy between $\mathcal{P}_\tau \rho_t $ and $\rho_{t+\tau}$ on the observables $\{\phi_k\}_{k=1}^K$ can be measured by the Wasserstein-2 metric as follows:
\begin{multline}\label{eq:wp}
    W_{\mathcal{P}_\tau}^{2} (\rho_{t},\rho_{t+\tau}) \coloneqq \\ \inf_{\pi\in\Pi(\mathcal{P}_{\tau} \rho_t,\rho_{t+\tau})} \int_{\IR^n\times \IR^n}{\bigl\|\Phi(x)-\Phi(y)\bigr\|}^2 \pi(\dd x,\dd y).
\end{multline}
Similarly, we can define the discrepancy between $\mathcal{K}_{\tau} \Phi$ and $\Phi$ under the Wasserstein-2 metric as
\begin{multline}\label{eq:wk}
    W_{\mathcal{K}_\tau}^{2} (\rho_{t},\rho_{t+\tau}) \coloneqq \\ \inf_{\pi\in\Pi(\rho_t,\rho_{t+\tau})} \int_{\IR^n\times \IR^n}{\| \mathcal{K}_{\tau} \Phi(x)-\Phi(y)\|}^2 \pi(\dd x,\dd y).
\end{multline}
The following theorem shows the equivalence between these two discrepancies.

\begin{theorem}
    For any set of observables $\{\phi_k\}_{k=1}^K \subset L_\infty$, we have
    $ W_{\mathcal{P}_\tau}^{2} (\rho_{t},\rho_{t+\tau}) = W_{\mathcal{K}_\tau}^{2} (\rho_{t},\rho_{t+\tau}).$
\end{theorem}
\begin{proof}
For any $\pi\in\Pi(\mathcal{P} \rho_t,\rho_{t+\tau})$, we first obtain
    \begin{align*}
    &\int_{\IR^n\times \IR^n}{\bigl\|\Phi(x)-\Phi(y)\bigr\|}^2 \pi(\dd x,\dd y) \\&=
    \int_{\IR^n} \mathcal{P}\rho_{k}(x) \int_{\IR^n} {\bigl\|\Phi(x)-\Phi(y)\bigr\|}^2 \pi(\dd y|x) \dd x
    \\&= \int_{\IR^n} \rho_{k}(x) \mathcal{K}\int_{\IR^n} {\bigl\|\Phi(x)-\Phi(y)\bigr\|}^2 \pi(\dd y|x) \dd x
    \\&= \int_{\IR^n} \rho_{k}(x) \int_{\IR^n} {\bigl\|\Phi(S(x))-\Phi(y)\bigr\|}^2 \pi(\dd y|S(x)) \dd x,
\end{align*}
where the first equality comes from the disintegration theorem, and the second is due to the duality between the PF operator $\mathcal{P}$ and the Koopman operator $\mathcal{K}$, and the third follows from the definition of Koopman operators.
Let $\pi_1(\dd x, \dd y) = \rho_{t}(x) \pi(\dd y|S(x)) \dd x$. It is clear that
\begin{equation*}
    \pi_1(\dd x, \IR^n) = \rho_{t}(x) \pi(\IR^n|S(x)) \dd x = \rho_{k}(x) \dd x,
\end{equation*}
which implies that $\pi_1$ has a marginal density $\rho_{k}$. On the other hand, for any measurable function $g\colon \IR^n \to \IR$,
\begin{align*}
    \int_{\IR^n\times \IR^n} g(y) \pi_1(\dd x,\dd y) &=
    \int_{\IR^n} \rho_{t}(x) \int_{\IR^n} g(y)  \pi(\dd y|S(x)) \dd x \\ & =
    \int_{\IR^n} \mathcal{P}\rho_{t}(x) \int_{\IR^n} g(y)  \pi(\dd y|x) \dd x \\ & =
    \int_{\IR^n\times \IR^n} g(y) \pi(\dd x,\dd y) \\ & =
    \int_{\IR^n} g(y) \rho_{t+\tau}(y) \dd y,
\end{align*}
where the last two equalities come from the fact that $\pi\in\Pi(\mathcal{P}\rho_k,\rho_{t+\tau})$. As a result, $\pi_1$ has a marginal density $\rho_{k}$, and thus, $\pi_1\in\Pi(\rho_k,\rho_{t+\tau})$. This gives rises to
    \begin{equation}\label{eq:w_inequality1}
                W_{\mathcal{P}_\tau}^{2} (\rho_{t},\rho_{t+\tau}) \geq W_{\mathcal{K}_\tau}^{2} (\rho_{t},\rho_{t+\tau}).
    \end{equation}
    On the other hand, for any $\pi_1\in\Pi(\rho_t,\rho_{t+\tau})$, we have
        \begin{align*}
    &\int_{\IR^n\times \IR^n}{\bigl\|\mathcal{K} \Phi(x)-\Phi(y)\bigr\|}^2 \pi_1(\dd x,\dd y) \\&=
    \int_{\IR^n} \rho_{k}(x) \mathcal{K} \int_{\IR^n} {\bigl\| \Phi(x)-\Phi(y)\bigr\|}^2 \pi_1(\dd y|S^{-1}(x)) \dd x
    \\&= \int_{\IR^n} \mathcal{P}\rho_{k}(x) \int_{\IR^n} {\bigl\|\Phi(x)-\Phi(y)\bigr\|}^2 \pi(\dd y|S^{-1}(x)) \dd x
    \\&= \int_{\IR^n} \mathcal{P}\rho_{k}(x) \int_{\IR^n} {\bigl\|\Phi(x)-\Phi(y)\bigr\|}^2 \pi(\dd y|S^{-1}(x)) \dd x.
\end{align*}
Let $\pi(\dd x, \dd y) = \mathcal{P}\rho_{t}(x) \pi_1(\dd y|S^{-1}(x)) \dd x$. A similar approach shows that $\pi\in\Pi(\mathcal{P} \rho_t,\rho_{t+\tau})$, and in turn
\begin{equation}\label{eq:w_inequality2}
                W_{\mathcal{P}_\tau}^{2} (\rho_{t},\rho_{t+\tau}) \leq W_{\mathcal{K}_\tau}^{2} (\rho_{t},\rho_{t+\tau}).
    \end{equation}
    The proof is completed by combining the inequalities \eqref{eq:w_inequality1} and \eqref{eq:w_inequality2}.
\end{proof}
\begin{remark}
\cite{karimi2022data} is among the earliest works to explicitly characterize flow properties and to address the problem via data-driven Perron--Frobenius operator learning, particularly under structured settings such as linear dynamics and affine dynamics on manifolds. Inspired by this initiative and motivated by the rapid development and broad adoption of modern generative modeling, we further develop the theory of \emph{Perron--Frobenius Operator Matching (PFOM)}. Our goal is to connect classical optimal mass/flow transport objectives with contemporary generative frameworks, including flow matching and diffusion-based models, within a unified operator-theoretic perspective.
\end{remark}
\subsection{Bregman-Divergence Guided PFOM}\label{sec:kl-uniqueness}
 
In PFOM we match the finite-step density evolution $\rho_{t+\tau}\approx\mathcal{P}_\tau\rho_t$,
yet training only accesses conditionals indexed by the data sample $X_1\sim q$, whose mixture
is the marginal, $\rho_{t+\tau}=\E_{X_1\sim q}\rho_{t+\tau}\cd{X_1}$. We ask the density
discrepancy $D$ to meet two requirements. \textbf{(P1)} \emph{(conditional--marginal
consistency)}: for every data law $q$ and every family of conditional densities
$\{\rho\cd{X_1}\}$ with marginal $\bar\rho=\E_{X_1\sim q}\rho\cd{X_1}$, and every density
$\sigma$,
\begin{align}\label{eq:P1}
 &\E_{X_1\sim q}\,D\bigl(\rho\cd{X_1}\,\big\|\,\sigma\bigr)=D(\bar\rho\,\|\,\sigma)+C,
\\ &C=C(q)\ \text{independent of }\sigma,
\end{align}
so the conditional loss differs from the marginal objective only by a parameter-free constant
and is thus an exact training surrogate. \textbf{(P2)} \emph{(reparametrization invariance)}:
for every nondegenerate coordinate change (diffeomorphism) $T$ with push-forward $\Tpush{T}$,
\begin{equation}\label{eq:P2}
 D\bigl(\Tpush{T}\rho\,\big\|\,\Tpush{T}\sigma\bigr)=D(\rho\,\|\,\sigma),
\end{equation}
so the aligned operator is intrinsic to the densities, not an artifact of the chosen (and, for
adaptive dictionaries, varying) observable coordinates. These two requirements single out the
Kullback--Leibler divergence.
 
\begin{theorem}

[KL is the unique consistent, coordinate-invariant discrepancy]\label{thm:kl}
Let $D(\rho\|\sigma)=\int\delta(\rho(x),\sigma(x))\,\dd x$ be a separable divergence with
$\delta\in C^2((0,\infty)^2)$, $\delta(s,s)=0$ and $\delta\ge0$. Then $D$ satisfies both
\eqref{eq:P1} and \eqref{eq:P2} if and only if $D=c\,\KL$ for some constant $c>0$.
\end{theorem}
 
\begin{proof}
\textbf{If.} Let $D=c\,\KL$. For \eqref{eq:P1}, with $\bar\rho=\E_{X_1\sim q}\rho\cd{X_1}$,
\begin{align}
  &\E_{X_1\sim q}\KL\bigl(\rho\cd{X_1}\big\|\sigma\bigr)
 \\=&\E_{X_1\sim q}\!\int\rho\cd{X_1}\log\frac{\rho\cd{X_1}}{\bar\rho}
+\E_{X_1\sim q}\!\int\rho\cd{X_1}\log\frac{\bar\rho}{\sigma}
 \\=&\underbrace{\E_{X_1\sim q}\KL\bigl(\rho\cd{X_1}\big\|\bar\rho\bigr)}_{=:C,\ \sigma\text{-free}}
+\,\KL(\bar\rho\|\sigma),   
\end{align}
because $\log(\bar\rho/\sigma)$ does not depend on $X_1$ and $\E_{X_1\sim q}\rho\cd{X_1}=\bar\rho$.
For \eqref{eq:P2}, change of variables under any diffeomorphism $T$ gives
$\KL(\Tpush{T}\rho\|\Tpush{T}\sigma)=\KL(\rho\|\sigma)$, the Jacobian cancelling because KL
depends on $(\rho,\sigma)$ only through $\rho$ and the ratio $\rho/\sigma$.
 
\textbf{Only if.} \emph{Step 1: \eqref{eq:P1}$\Rightarrow$ Bregman.} Specialize $q$ to the
two-point law placing mass $\lambda$ on $x_1^{0}$ and $1-\lambda$ on $x_1^{1}$, and write
$\rho_0=\rho\cd{x_1^{0}}$, $\rho_1=\rho\cd{x_1^{1}}$, so $\bar\rho=\lambda\rho_0+(1-\lambda)\rho_1$.
Then \eqref{eq:P1} states that
$\lambda D(\rho_0\|\sigma)+(1-\lambda)D(\rho_1\|\sigma)-D(\bar\rho\|\sigma)$ is independent of
$\sigma$. Pointwise (by separability), the map
$s\mapsto\lambda\delta(r_0,s)+(1-\lambda)\delta(r_1,s)-\delta\bigl(\lambda r_0+(1-\lambda)r_1,s\bigr)$
is constant for all $r_0,r_1>0$ and $\lambda\in(0,1)$; hence for any $s,s'$ the function
$r\mapsto\delta(r,s)-\delta(r,s')$ has vanishing Jensen gap, i.e.\ is affine. Fix a reference
$s_0$ and set $v(r)\coloneqq\delta(r,s_0)$; then $\delta(r,s)=v(r)+a(s)\,r+b(s)$. Imposing
$\delta(s,s)=0$ and $\partial_r\delta(r,s)|_{r=s}=0$ (as $r=s$ minimizes $\delta(\cdot,s)$)
gives $a(s)=-v'(s)$ and $b(s)=s\,v'(s)-v(s)$, hence
\[
 \delta(r,s)=v(r)-v(s)-v'(s)(r-s),
\]
i.e.\ $D$ is a separable Bregman divergence with potential $v$, convex since $\delta\ge0$.
 
\emph{Step 2: \eqref{eq:P2}$\Rightarrow$ KL.} Write $B_v\coloneqq\delta$. Take the scaling
$T_\mu(x)=\mu x$ on $\IR^n$ ($J=\mu^n$), whose push-forward is
$(\Tpush{T_\mu}\rho)(y)=\rho(y/\mu)/J$. Then \eqref{eq:P2} and a change of variables give, for
all densities, $\int J\,B_v(\rho/J,\sigma/J)\,\dd x=\int B_v(\rho,\sigma)\,\dd x$, hence
pointwise $J\,B_v(r/J,s/J)=B_v(r,s)$ for all $J>0$: $B_v$ is positively homogeneous of degree
one, $B_v(\lambda r,\lambda s)=\lambda B_v(r,s)$. Differentiating twice in $r$ gives
$\lambda^2 v''(\lambda r)=\lambda v''(r)$, so $\lambda\,v''(\lambda r)=v''(r)$; at $r=1$,
$v''(\lambda)=c/\lambda$ with $c=v''(1)>0$. Integrating, $v(r)=c\,r\log r$ up to affine terms,
whence $D=c\,\KL$.
\end{proof}
 
\begin{remark}[Significance]
\textnormal{(i)} Property \eqref{eq:P1} is the marginalization (compensation) identity that
makes conditional, sample-based training an exact surrogate for the marginal objective; by
Step~1 it characterizes the \emph{entire} Bregman class, so it alone does \emph{not} select
KL---squared error ($v(r)=\tfrac12 r^2$) also satisfies it, with constant
$C=\E_{X_1\sim q}\KL(\rho\cd{X_1}\|\bar\rho)$ specializing, for KL, to the mutual information
between $X_1$ and the sample. \textnormal{(ii)} Property \eqref{eq:P2} is the intrinsic,
coordinate-free requirement, essential because PFOM matches operators \emph{through} a chosen
and adaptive observable dictionary; squared error violates it (least squares is not invariant
under nonlinear reparametrization). \textnormal{(iii)} Only the \emph{conjunction} forces KL,
equivalently the unique Bregman divergence that is simultaneously an $f$-divergence. The
argument is exact for finite $\tau$ and uses neither a small-$\tau$ expansion nor the
infinitesimal generators. \textnormal{(iv)} It also corrects the naive identity
$D(\rho_{t+\tau}\|\mathcal{P}_\tau\rho_t)=\E_{X_1\sim q}[D(\text{conditionals})]$: equality holds
only with the parameter-free constant $C$ and the \emph{shared} marginal $\mathcal{P}_\tau\rho_t$
as the second argument; the constant-free, per-$X_1$ form fails even for KL.
\end{remark}


\subsection{Connections with Flow Matching}
Flow matching and diffusion-style training arise as the \emph{Gaussian reduction} of PF
matching: when the one-step conditional transitions are Gaussian with a shared noise schedule,
a single least-squares loss controls \emph{both} the marginal Wasserstein-2 and the marginal
KL PF objectives. The mechanism is shared---both $W_2^2$ and $\KL$ are jointly convex, and for
two Gaussians with the same covariance both reduce to the squared distance between means.
 
\begin{theorem}

[Flow matching is the common surrogate for marginal PF matching]
\label{thm:fm}
Let $Z=X_1\sim q$, and assume conditionally Gaussian one-step transitions with a shared
isotropic covariance and a shared drift field $f^\theta_t$,
\begin{align}
   &\rho_{t+\tau}\cd Z=\N\!\bigl(\mu_t(Z),\,g_t^2\tau I_d\bigr),
 \\&\Pmod\rho_t\cd Z=\N\!\bigl(f^\theta_t(X_t),\,g_t^2\tau I_d\bigr), 
\end{align}
with $g_t>0$ fixed (independent of $\theta$) and $X_t$ the current state on the conditional
path. Write the flow-matching loss
\[
 \LFM(\theta)\;\coloneqq\;\E_{Z}\bigl\|\mu_t(Z)-f^\theta_t(X_t)\bigr\|^2 .
\]
Then the marginal Wasserstein and KL PF objectives obey
\begin{align}\label{eq:fm-bounds}
 &W_2^2\!\bigl(\Pmod\rho_t,\rho_{t+\tau}\bigr)\;\le\;\LFM(\theta),
\\&
 \KL\!\bigl(\rho_{t+\tau}\,\big\|\,\Pmod\rho_t\bigr)\;\le\;\frac{1}{2g_t^2\tau}\,\LFM(\theta),
\end{align}
and $\LFM$ equals the denoising least-squares (sample) loss up to a $\theta$-free constant,
\begin{equation}\label{eq:fm-nll}
 \LFM(\theta)=\E_{Z,X_{t+\tau}}\bigl\|X_{t+\tau}-f^\theta_t(X_t)\bigr\|^2-d\,g_t^2\tau .
\end{equation}
Consequently the flow-matching loss is a common offline surrogate for both marginal objectives,
and its unique minimizer over drift fields is the marginal drift
$f_t^{\star}(x)=\E\!\bigl[\mu_t(Z)\,\big|\,X_t=x\bigr]$.
\end{theorem}
 
\begin{proof}
 For two Gaussians with common covariance
$\Sigma=g_t^2\tau I_d$, the Bures term vanishes and
\begin{align}
  &W_2^2\!\bigl(\N(\mu,\Sigma),\N(m,\Sigma)\bigr)=\|\mu-m\|^2,\\&\KL\!\bigl(\N(\mu,\Sigma)\,\big\|\,\N(m,\Sigma)\bigr)=\frac{\|\mu-m\|^2}{2g_t^2\tau}.  
\end{align}
Applying these to the conditionals and taking $\E_Z$,
\begin{align}
    \label{eq:cond-exact}
 &\E_Z\,W_2^2\!\bigl(\Pmod\rho_t\cd Z,\rho_{t+\tau}\cd Z\bigr)=\LFM(\theta),
\\&\E_Z\,\KL\!\bigl(\rho_{t+\tau}\cd Z\,\big\|\,\Pmod\rho_t\cd Z\bigr)=\frac{\LFM(\theta)}{2g_t^2\tau}.
\end{align}
On the other hand, both $W_2^2$ and $\KL$ are
jointly convex in their two arguments, so for any mixing law,
$D\!\bigl(\E_Z\alpha_Z\,\|\,\E_Z\beta_Z\bigr)\le\E_Z D(\alpha_Z\|\beta_Z)$. Since
$\rho_{t+\tau}=\E_Z\rho_{t+\tau}\cd Z$ and $\Pmod\rho_t=\E_Z\Pmod\rho_t\cd Z$, applying this to
$D\in\{W_2^2,\KL\}$ and combining with \eqref{eq:cond-exact} yields the two bounds
\eqref{eq:fm-bounds}.
  With $X_{t+\tau}\cd Z=\mu_t(Z)+g_t\sqrt{\tau}\,\varepsilon$,
$\varepsilon\sim\N(0,I_d)$ independent of $(Z,X_t)$,
\[
 \E\!\bigl[\|X_{t+\tau}-f^\theta_t(X_t)\|^2\,\big|\,Z\bigr]
 =\|\mu_t(Z)-f^\theta_t(X_t)\|^2+d\,g_t^2\tau ;
\]
taking $\E_Z$ gives \eqref{eq:fm-nll}, whose additive constant $d\,g_t^2\tau$ is independent of
$\theta$.
  Writing $\LFM(\theta)=\E_{X_t}\E_{Z\mid X_t}\|\mu_t(Z)-f^\theta_t(X_t)\|^2$
and minimizing over fields pointwise in $x$, the optimum is the conditional mean
$f_t^{\star}(x)=\E[\mu_t(Z)\mid X_t=x]$.
\end{proof}
 
 
\begin{remark}[Why PFOM is strictly more general]
The reduction hinges on the Gaussian, shared-covariance, shared-drift assumption: it confines
the learned evolution to the local form $X_{t+\tau}\!\mid\! Z=\mu_t(Z)+g_t\sqrt\tau\,\varepsilon$,
a first/second-order differential model. In general the ground-truth one-step evolution need
not admit such a representation---it may carry higher-order, non-local, or multi-step transport
that a purely infinitesimal velocity/diffusion model cannot express. This is salient in, e.g.,
imitation learning, where the teacher flow may be strongly multimodal or generated by a complex
decision process. Matching the full finite-step PF operator, as in PFOM, retains these effects
that flow matching discards.
\end{remark}
\subsection{Nesterov Momentum Acceleration for Generation}
\label{Nesterov}

We incorporate Nesterov’s acceleration \citep{nesterov2018lectures} at the \emph{observable} level so that evaluation is performed at a look-ahead point. Let $\{\phi_k\}_{k\ge 1}$ be an observable basis, and denote the Koopman step by $\mathcal{K}_\tau^\theta$. Define the extrapolated (look-ahead) observable
\begin{equation*}
    \psi_k(x_t) = \phi_k(x_t) + \eta_t \bigl(\phi_k(x_t)-\phi_k(x_{t-\tau})\bigr),
\quad \eta_t\in[0,1).
\end{equation*}
We replace $\phi_k(x_t)$ with a momentum look-ahead on the input observables:
\begin{multline}
\label{eq:nesterov-train}
 \mathcal{L}_{\text{KPM-Nes}}(\theta)
=
\sum_{k=1}^{K}
\mathbb{E}_{X_{0}\sim p,X_1\sim q}
\bigl\|
\phi_k\bigl(x_{t+\tau}(X_0,X_1)\bigr)\bigr.
\\{\bigl. -
\mathcal{K}_\tau^\theta \psi_k\bigl(x_t(X_0,X_1)\bigr)
\bigr\|}^2.
\end{multline}
In coordinates, i.e., $\phi_k(x)=x^{(k)}$, \eqref{eq:nesterov-train} reduces to the vector form
\begin{multline}\label{eq:nesterov-train-coords}
  \mathbb{E}_{X_{0}\sim p,X_1\sim q}
\Big\|
x_{t+\tau}(X_0,X_1)
-
\hat{\mathcal{K}}_\tau^\theta \bigl( x_t(X_0,X_1)\\ + \eta_t(x_t(X_0,X_1)-x_{t-\tau}(X_0,X_1))
\bigr) \Bigr\|^2,
\end{multline}
where $\hat{\mathcal{K}}_\tau^\theta$ is constructed as the Koopman operator.
Given a trained $\mathcal{K}_\tau^\theta$, we propagate with a look-ahead state:
\begin{subequations}\label{eq:nesterov-sample}
\begin{align}
y_t &= x_t + \eta_t\,(x_t - x_{t-\tau}), \\
x_{t+\tau} &= \hat{\mathcal{K}}_\tau^\theta (y_t),
\quad x_0,x_\tau \sim \mathcal{N}(0,I).
\end{align}
\end{subequations}
In light of the Nesterov momentum method in optimization theory, we here introduce formally the following  Nesterov-KPM training and sampling algorithms.

\begin{algorithm}[tb]
\caption{Nesterov-KPM Training (mini-batch)}
\begin{algorithmic}[1]
\State \textbf{Inputs:} step $\tau$, momentum $\eta$, bridge $x_s(X_1,X_0)$
\For{batches of pairs $(X_0^{(i)},X_1^{(i)})$}
 \State build $x_{t-\tau}^{(i)},x_t^{(i)},x_{t+\tau}^{(i)}$ from the bridge
 \State $y_t \gets x_t^{(i)} + \eta\big(x_t^{(i)}-x_{t-\tau}^{(i)}\big)$
 \State $\hat z_t \gets \hat{\mathcal{K}}_\tau^\theta \big(y_t\big)$
 \State $\mathcal{L}_{\mathrm{KPM}\text{-}\mathrm{Nes}} \gets \frac{1}{B}\sum_i \|x_{t+\tau}-\hat z_t\|_2^2$
 \State Update $\theta$ by gradient descent on $\mathcal{L}_{\mathrm{KPM}\text{-}\mathrm{Nes}}$
\EndFor
\end{algorithmic}
\end{algorithm}\label{alg1}

\begin{algorithm}[tb]
\caption{Nesterov-KPM Sampling}
\begin{algorithmic}[1]
\State Initialize $x_0,x_{-\tau} \sim \mathcal{N}(0,I)$, set $t\!=\!0$
\While{$t<1$}
 \State $y_t \gets x_t + \eta (x_t - x_{t-\tau})$
 \State $x_{t+\tau} \gets \hat{\mathcal{K}}_\tau^\theta (y_t)$
 \State $t \gets t+\tau$
\EndWhile
\end{algorithmic}
\end{algorithm}\label{alg2}


\section{Numerical Simulations}\label{simu}

In this section, we follow the training loss in the flow matching, representing the Koopman operator using a deep neural network parameterized by $\theta$, that is,
\begin{multline}
 \min_\theta\mathbb E_{X_{0}\sim p,X_1\sim q,t\sim\mathrm{U}[0,1]}\big\|x_{t+\tau}(X_1,X_0) \\ - \mathrm{NN}_\theta(t, x_t(X_1,X_0))\big\|^2.
 \end{multline}
After getting the optimal parameter $\theta$, we do the following iteration over time $t$:
$\hat{X}_{t+\tau}=\mathrm{NN}_\theta(t,\hat{X}_t), \quad\hat X_0\sim N(0,I).
$
Some important hyperparameters are listed in Table \ref{tb:margins}.
\begin{table}[tb]
\begin{center}
\caption{Parameter settings}\label{tb:margins}
\vspace*{-1ex}
\begin{tabular}{cccc}
Names & Values\\\hline
Number of samples & $1000$\\
Learning rate; Training epochs & $0.001$; $5000$\\
Hidden dimension; Hidden layers & $128$; $4$\\
$\tau$; $\eta$; Optimizer & $0.05$; $0.25$; Adam \\ 
Implementation framework & Python / PyTorch \\
GPU; CUDA acceleration & NVIDIA A100; Enabled\\
\hline
\end{tabular}
\end{center}
\end{table}
Fig.~\ref{fig:image1} shows the generated and original samples of GMM model and the Two-Moon model, respectively.
\begin{figure}[tb]
\begin{subfigure}{0.22\textwidth}
\includegraphics[width=\textwidth]{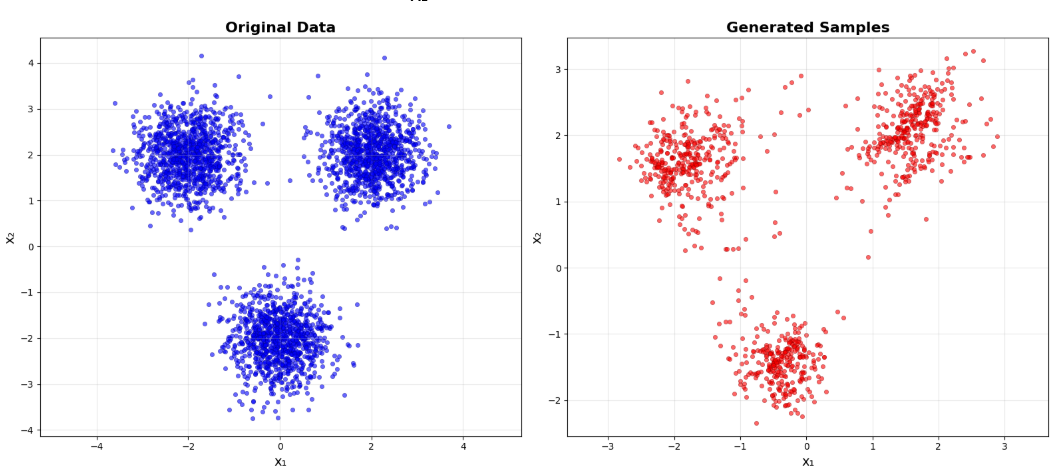}
\end{subfigure}
\begin{subfigure}{0.22\textwidth}
\includegraphics[width=\textwidth]{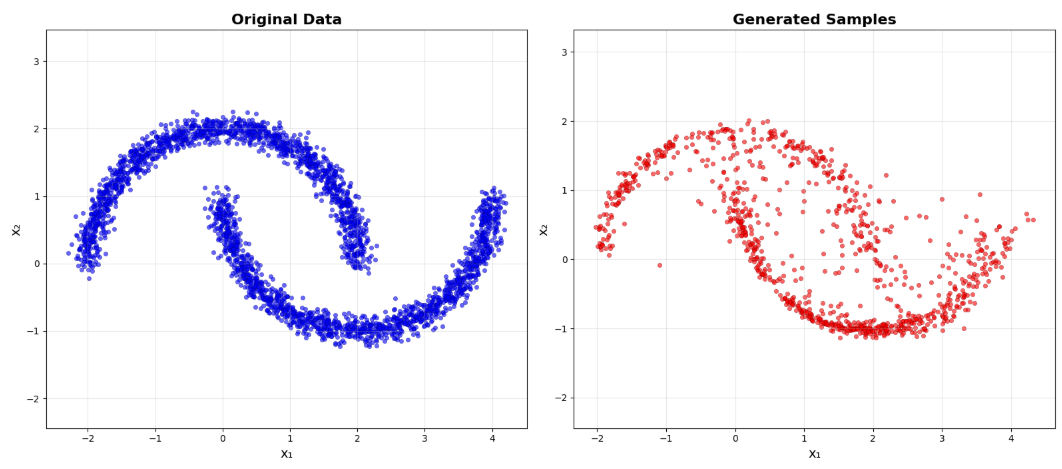}
\end{subfigure}
\vspace*{-3ex}
\caption{\small Original Samples (Blue) from GMM (Left) / Two-Moon (Right) Models and Generated Samples (Red) by PFOM.}
\label{fig:image1}\vspace{5mm}
\end{figure}
Moreover, we train with the Nesterov momentum loss in Algorithm \ref{alg1} and generate samples via Algorithm \ref{alg2} on the GMM benchmark. Figure~\ref{fig:image2} compares the \emph{rates of decrease} in KL divergence, Wasserstein-2, and maximum mean discrepancy (MMD) between standard Koopman path matching and its Nesterov-accelerated variant. The Nesterov method consistently achieves faster and better convergence. The reported curves correspond to a representative run; multi-seed evaluation is left for future work.
\begin{figure}[tb]
\begin{subfigure}{0.45\textwidth}
\includegraphics[width=\textwidth]{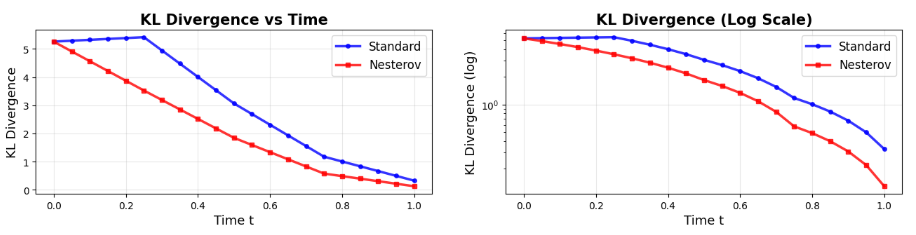}
\end{subfigure}\vspace*{-1ex}
\begin{subfigure}{0.45\textwidth}
\includegraphics[width=\textwidth]{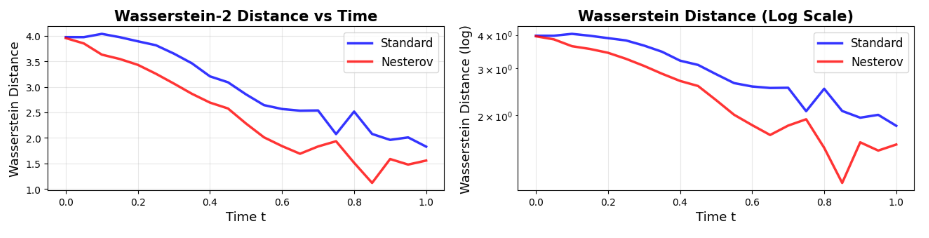}
\end{subfigure}
\vspace*{-4ex}
\begin{subfigure}{0.45\textwidth}
\includegraphics[width=\textwidth]{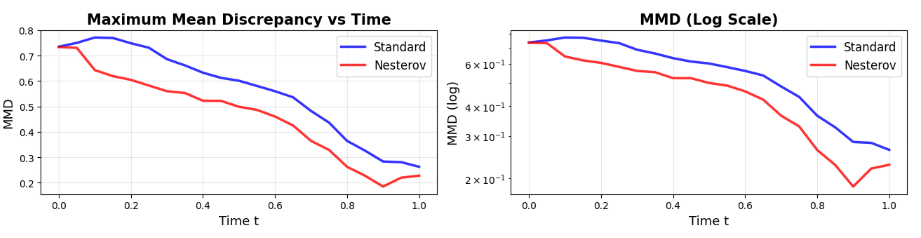}
\end{subfigure}
\vspace*{1ex}
\caption{\small Comparison of KL-divergence (First row)/ $W_2$ metric (Second row)/maximum mean discrepancy (Third row) decreasing rate.}
\label{fig:image2}
\end{figure}
Moreover, we also show in Fig. \ref{fig:image_process} the generating process for GMM/Two-moons model of our Nesterov-KPM sampling method for demonstration.
\begin{figure}[tb]
\begin{subfigure}{0.45\textwidth}
\includegraphics[width=\textwidth]{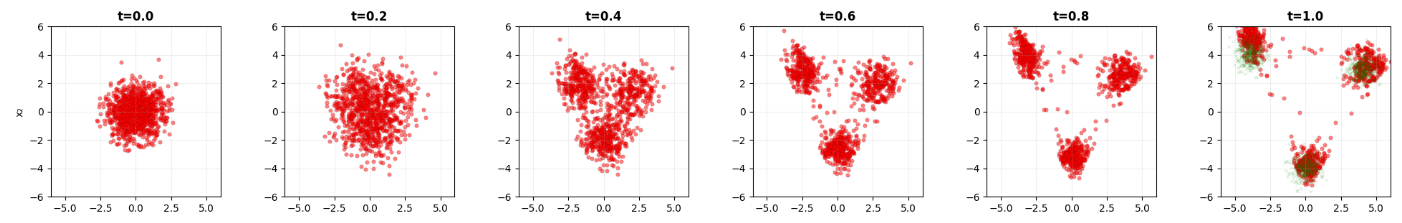}
\end{subfigure}
\begin{subfigure}{0.45\textwidth}
\includegraphics[width=\textwidth]{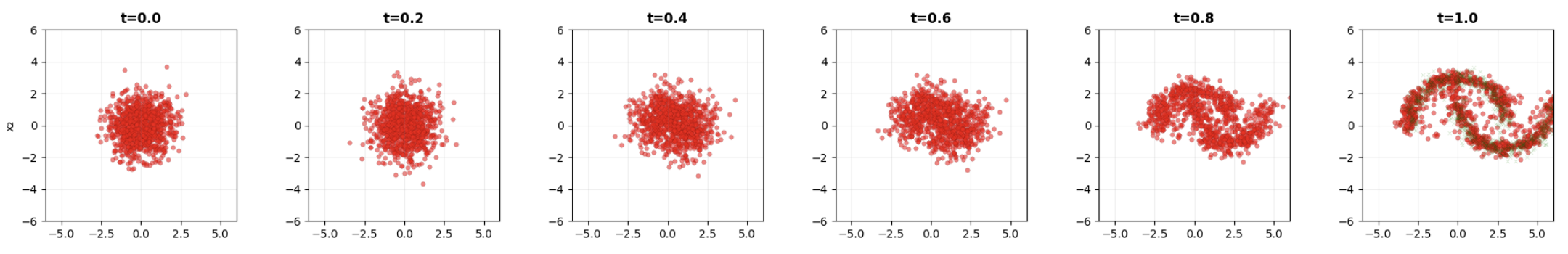}
\end{subfigure}
\vspace*{-3ex}
\caption{\small Generating process of our Nesterov-KPM Sampling.}
\label{fig:image_process}
\end{figure}

\vspace{5mm}
\section{Conclusions}\label{conclu}

We have introduced Perron--Frobenius Operator Matching(PFOM), an operator-theoretic framework that connects density-level Perron--Frobenius evolution, Koopman path matching,and sample-conditioned generative training. We showed that theKullback--Leibler divergence has a distinguished role amongseparable Bregman divergences in preserving the alignment betweenmarginal density objectives and conditional losses. We alsodeveloped a Nesterov-type inertial variant, which improves theempirical convergence behavior of the Koopman path-matchingimplementation on Gaussian mixture and two-moon benchmarks.The present experiments serve as low-dimensional proof-of-conceptvalidations. Future work will focus on higher-dimensionalbenchmarks, adaptive observable dictionaries, latent-space imagemodeling, and controlled PFOM formulations with explicit input orfeedback dependence. More systematic empirical evaluation,including multi-seed robustness, uncertainty bands, and comparisonswith standard flow-matching and diffusion baselines, will also beimportant for assessing the practical scalability of the proposedapproach.

\smallskip

\bibliography{ifacconf} 

@article{lipman2022flow,
  title={Flow matching for generative modeling},
  author={Lipman, Yaron and Chen, Ricky TQ and Ben-Hamu, Heli and Nickel, Maximilian and Le, Matt},
  journal={arXiv preprint arXiv:2210.02747},
  year={2022}
}

@article{proctor2016dynamic,
  title={Dynamic mode decomposition with control},
  author={Proctor, Joshua L and Brunton, Steven L and Kutz, J Nathan},
  journal={SIAM Journal on Applied Dynamical Systems},
  volume={15},
  number={1},
  pages={142--161},
  year={2016},
  publisher={SIAM}
}

@book{ross1995stochastic,
  title={Stochastic processes},
  author={Ross, Sheldon M},
  year={1995},
  publisher={John Wiley \& Sons}
}

@book{lemmens2012nonlinear,
  title={Nonlinear Perron-Frobenius Theory},
  author={Lemmens, Bas and Nussbaum, Roger},
  volume={189},
  year={2012},
  publisher={Cambridge University Press}
}

@article{brunton2016discovering,
  title={Discovering governing equations from data by sparse identification of nonlinear dynamical systems},
  author={Brunton, Steven L and Proctor, Joshua L and Kutz, J Nathan},
  journal={Proceedings of the national academy of sciences},
  volume={113},
  number={15},
  pages={3932--3937},
  year={2016},
  publisher={National Academy of Sciences}
}

@book{nesterov2018lectures,
  title={Lectures on convex optimization},
  author={Nesterov, Yurii and others},
  volume={137},
  year={2018},
  publisher={Springer}
}

@book{eisner2015operator,
  title={Operator theoretic aspects of ergodic theory},
  author={Eisner, Tanja and Farkas, B{\'a}lint and Haase, Markus and Nagel, Rainer},
  volume={272},
  year={2015},
  publisher={Springer}
}

@article{yang2023diffusion,
  title={Diffusion models: A comprehensive survey of methods and applications},
  author={Yang, Ling and Zhang, Zhilong and Song, Yang and Hong, Shenda and Xu, Runsheng and Zhao, Yue and Zhang, Wentao and Cui, Bin and Yang, Ming-Hsuan},
  journal={ACM computing surveys},
  volume={56},
  number={4},
  pages={1--39},
  year={2023},
  publisher={ACM New York, NY, USA}
}

@article{katz2022verification,
  title={Verification of image-based neural network controllers using generative models},
  author={Katz, Sydney M and Corso, Anthony L and Strong, Christopher A and Kochenderfer, Mykel J},
  journal={Journal of Aerospace Information Systems},
  volume={19},
  number={9},
  pages={574--584},
  year={2022},
  publisher={American Institute of Aeronautics and Astronautics}
}

@article{cui2025gencontrol,
  title={GenControl: Generative AI-Driven Autonomous Design of Control Algorithms},
  author={Cui, Chenggang and Liu, Jiaming and Hui, Peifeng and Lin, Pengfeng and Zhang, Chuanlin},
  journal={arXiv preprint arXiv:2506.12554},
  year={2025}
}

@article{ho2020denoising,
  title={Denoising diffusion probabilistic models},
  author={Ho, Jonathan and Jain, Ajay and Abbeel, Pieter},
  journal={Advances in neural information processing systems},
  volume={33},
  pages={6840--6851},
  year={2020}
}

@book{oppenheim1997signals,
  title={Signals \& systems},
  author={Oppenheim, Alan V and Willsky, Alan S and Nawab, Syed Hamid},
  year={1997},
  publisher={Pearson Educaci{\'o}n}
}

@book{van1992stochastic,
  title={Stochastic processes in physics and chemistry},
  author={Van Kampen, Nicolaas Godfried},
  volume={1},
  year={1992},
  publisher={Elsevier}
}

@book{rolski2009stochastic,
  title={Stochastic processes for insurance and finance},
  author={Rolski, Tomasz and Schmidli, Hanspeter and Schmidt, Volker and Teugels, Jozef L},
  year={2009},
  publisher={John Wiley \& Sons}
}

@book{lasota2013chaos,
  title={Chaos, fractals, and noise: stochastic aspects of dynamics},
  author={Lasota, Andrzej and Mackey, Michael C},
  volume={97},
  year={2013},
  publisher={Springer Science \& Business Media}
}

@incollection{risken1989fokker,
  title={Fokker-planck equation},
  author={Risken, Hannes},
  booktitle={The Fokker-Planck equation: methods of solution and applications},
  pages={63--95},
  year={1989},
  publisher={Springer}
}

@article{holderrieth2024generator,
  title={Generator matching: Generative modeling with arbitrary markov processes},
  author={Holderrieth, Peter and Havasi, Marton and Yim, Jason and Shaul, Neta and Gat, Itai and Jaakkola, Tommi and Karrer, Brian and Chen, Ricky TQ and Lipman, Yaron},
  journal={arXiv preprint arXiv:2410.20587},
  year={2024}
}

@article{li2017extended,
  title={Extended dynamic mode decomposition with dictionary learning: A data-driven adaptive spectral decomposition of the Koopman operator},
  author={Li, Qianxiao and Dietrich, Felix and Bollt, Erik M and Kevrekidis, Ioannis G},
  journal={Chaos: An Interdisciplinary Journal of Nonlinear Science},
  volume={27},
  number={10},
  year={2017},
  publisher={AIP Publishing}
}

@article{lipman2024flow,
  title={Flow matching guide and code},
  author={Lipman, Yaron and Havasi, Marton and Holderrieth, Peter and Shaul, Neta and Le, Matt and Karrer, Brian and Chen, Ricky TQ and Lopez-Paz, David and Ben-Hamu, Heli and Gat, Itai},
  journal={arXiv preprint arXiv:2412.06264},
  year={2024}
}

@article{karimi2022data,
  title={Data-driven approximation of the Perron-Frobenius operator using the Wasserstein metric},
  author={Karimi, Amirhossein and Georgiou, Tryphon T},
  journal={IFAC-PapersOnLine},
  volume={55},
  number={30},
  pages={341--346},
  year={2022},
  publisher={Elsevier}
}








\end{document}